\documentclass{article}

\usepackage{PRIMEarxiv}

\usepackage[utf8]{inputenc} 
\usepackage[T1]{fontenc}    
\usepackage{hyperref}       
\usepackage{url}            
\usepackage{booktabs}       
\usepackage{amsfonts}       
\usepackage{nicefrac}       
\usepackage{microtype}      
\usepackage{lipsum}
\usepackage{fancyhdr}       
\usepackage{graphicx}       
\graphicspath{{media/}}     

\title{StructZO: Structured Zeroth-Order Optimization for Memory-Efficient LLM Fine-Tuning
}

\usepackage{multirow}
\usepackage{graphicx}
\usepackage{amsmath}
\usepackage{amssymb}
\usepackage{mathtools}
\usepackage{amsthm}
\usepackage{multirow}
\usepackage{subcaption}

\author{
  Xuanlin Yang \\
  Peking University \\
  \texttt{2401112159@stu.pku.edu.cn} \\
   \And
  Jia Li \\
  Beijing Normal University \\
  \texttt{jiali.gm@gmail.com} \\
  \And
  Zhouchen Lin \\
  Peking University \\
  ZLIN@pku.edu.cn
}

\usepackage{bm}
\usepackage{algorithm}
\usepackage{algpseudocode}
\newtheorem{theorem}{Theorem}

\newtheorem{prop}[theorem]{Proposition}
\newtheorem{assumption}[theorem]{Assumption}

\newtheorem{remark}[theorem]{Remark}

\newcommand{\comment}[1]{}

\newcommand{\bx}{\mathbf{x}}

\newcommand{\bz}{\mathbf{z}}

\newcommand{\bI}{\mathbf{I}}

\newcommand{\bP}{\mathbf{P}}

\newcommand{\bzero}{\mathbf{0}}

\newcommand{\<}{\left\langle}
\renewcommand{\>}{\right\rangle}

\begin{document}
\maketitle

\begin{abstract}
Fine-tuning large language models at scale is increasingly limited by the memory cost of backpropagation. Zeroth-Order (ZO) optimizers reduce this burden by relying only on forward function evaluations, but they still exhibit a substantial performance gap to first-order (FO) optimizers such as Adam. In this work, we argue that this gap does not arise solely from the absence of explicit gradients: FO fine-tuning also benefits from informative low-rank optimization geometry, large-scale signal aggregation, and well-conditioned update directions. Motivated by this perspective, we introduce \textbf{StructZO}, a \textbf{struct}ured \textbf{Z}eroth-\textbf{O}rder optimization framework for memory-efficient LLM fine-tuning. StructZO performs large-query ZO estimation within an activation-induced informative subspace and applies spectral control to the resulting reduced-space update before mapping it back to the original parameters. 
This unified update procedure enables ZO optimization to aggregate useful directional signals in a compact, geometry-aware space while avoiding the activation and optimizer-state memory costs of backpropagation. 
We provide a theoretical analysis showing that, when the induced subspace captures concentrated optimization-relevant directions, StructZO reduces the effective ZO estimation dimension from $mn$ to $mr$, with the remaining projection error governed by the spectral tail of the activation Gram matrix. Empirically, across OPT-series and Llama3.1-8B on diverse downstream fine-tuning tasks, StructZO improves over the strongest existing ZO baselines by up to $8.7\%$ relative on average across model settings and up to $26.9\%$ relative on individual tasks, while remaining within $1.5\%$ of FO-Adam on average. These results suggest that in the evaluated fine-tuning settings, structured and spectrally controlled ZO updates can substantially narrow the gap between ZO and FO optimization.
\end{abstract}
\section{Introduction}
\label{seq:intro}
Fine-tuning large language models (LLMs) has become a central component in adapting large pretrained models to downstream tasks \cite{raffel2020exploring}. While first-order (FO) optimizers such as Adam \cite{DBLP:journals/corr/KingmaB14} and AdamW \cite{DBLP:conf/iclr/LoshchilovH19} are commonly used in fine-tuning, they require full backpropagation through the model and therefore incur substantial memory overhead from storing intermediate activations and optimizer states. These limitations have driven increasing interest in parameter- and memory-efficient fine-tuning methods \cite{DBLP:conf/iclr/HuSWALWWC22,ding2023parameter,DBLP:conf/icml/Zhao0CWAT24,DBLP:conf/acl/LiL20}. Zeroth-Order (ZO) optimization \cite{nesterov2017random}, exemplified by MeZO \cite{malladi2023fine}, provides a promising alternative by replacing gradient computation with forward-pass-only perturbation estimates. As a result, MeZO reduces memory consumption close to the inference level, making full-parameter LLM fine-tuning feasible under tight memory budgets.

Despite this memory advantage, ZO fine-tuning still suffers from slow convergence and a persistent performance gap to first-order methods, largely due to noisy estimates in high-dimensional and ill-conditioned parameter spaces. Since MeZO, recent work has improved ZO fine-tuning from several perspectives, including low-rank structure or subspace \cite{DBLP:conf/iclr/ChenZCYW25,DBLP:journals/corr/abs-2601-17261,DBLP:journals/corr/abs-2602-17155}, alternative sampling distributions \cite{ma2025revisiting,dang2025fzoo}, momentum and variance-reduction techniques \cite{yuepseuzo,DBLP:conf/icml/GautamPZRH24}, multi-sample estimators \cite{DBLP:conf/emnlp/PengLWC25}, and Hessian- or curvature-informed updates \cite{DBLP:conf/iclr/ZhaoDYD0T25,zhao2025helene}. These advances show that the naive full-space random estimator is far from optimal. However, most existing methods still focus primarily on improving the gradient estimation itself, while remaining limited by query efficiency, search-space quality, or update conditioning. In contrast, FO fine-tuning benefits from more than gradient access alone: it also exploits informative optimization geometry, aggregates rich derivative signals, and applies well-conditioned update transformations. This motivates the following question: \emph{Can we design a forward-only ZO optimizer that recovers these structural benefits of FO fine-tuning without incurring the memory cost of backpropagation?}

In this paper, we propose \textbf{StructZO}, a \textbf{struct}ured \textbf{Z}eroth-\textbf{O}rder optimization framework that improves forward-only LLM fine-tuning by performing large-query directional signal aggregation inside activation-induced informative subspaces, thereby converting additional ZO queries into more useful low-dimensional update estimates without storing activations for backpropagation. Since these subspaces are intended to approximate the right-side geometry of the gradient, StructZO applies Muon-style spectral shaping directly in the reduced coordinates before mapping the update back to the original parameter space, producing well-conditioned directions that better reflect FO-style optimization behavior.

Existing works address these factors in isolation, while StructZO provides the first \textit{coupled design}: the informative reduced space determines where ZO queries are spent, large-query estimation aggregates signals within this space, and spectral shaping is applied in the same coordinates before the update is lifted back. This coupling changes the role of multi-query ZO from noisy full-space gradient estimation to structured signal aggregation, making it possible to analyze the resulting variance--bias trade-off. Empirically, StructZO consistently improves over existing ZO baselines under comparable forward-query budgets and narrows the gap to FO-Adam in our evaluated fine-tuning settings.

Our contributions are summarized as follows:
\begin{itemize}
    \item We introduce StructZO, a unified structured ZO framework that performs estimation and update conditioning in an informative reduced space, recovering several key optimization benefits of FO fine-tuning without backpropagation.

    \item We show that this structured update space makes large-query ZO estimation effective in practice: additional queries are aggregated along more useful low-dimensional directions, while the resampling trick avoids storing full perturbation directions.

    \item We provide an MSE analysis that characterizes the variance--bias trade-off of StructZO. The result shows that the ZO estimation dimension is reduced from $mn$ to $mr$, while the remaining projection bias is controlled by the spectral tail of the activation Gram matrix.

    \item Experiments on OPT and Llama3.1 models show that StructZO consistently outperforms existing ZO baselines, achieves lower training loss, scales to larger models, and remains close to FO-Adam while preserving the memory advantage of forward-only optimization. For example, StructZO improves the average score on Llama3.1-8B from $76.3$ to $82.9$ over the strongest ZO baseline, and improves CB from $73.2$ to $92.9$.
\end{itemize}

\section{Related Work}
\label{seq:related}
\paragraph{Fine-tuning LLMs.}
Fine-tuning large language models with first-order optimizers such as Adam and AdamW is effective but memory-intensive. This motivates parameter- and memory-efficient fine-tuning methods, including adapters \cite{houlsby2019parameter}, prefix tuning \cite{DBLP:conf/emnlp/LesterAC21,DBLP:conf/acl/LiL20}, LoRA \cite{DBLP:conf/iclr/HuSWALWWC22}, and low-rank optimizer-memory reduction methods such as GaLore \cite{DBLP:conf/icml/Zhao0CWAT24, DBLP:conf/icml/HeLHC025}. Although these methods reduce the number of trainable parameters or optimizer states, they still rely on backpropagation and therefore retain nontrivial activation-memory cost.

\paragraph{Zeroth-Order Optimization.}
Zeroth-order optimization estimates update directions using only objective function values and has been widely studied in adversarial learning \cite{chen2017zoo, ilyas2018black}, reinforcement learning \cite{vemula2019contrasting}, and automated ML \cite{wang2022zarts}. Recently, MeZO \cite{malladi2023fine} showed that ZO optimization can fine-tune LLMs with memory close to inference, making it attractive under limited hardware budgets. Recent works improve ZO fine-tuning through low-rank and structured perturbation \cite{DBLP:conf/iclr/ChenZCYW25,DBLP:journals/corr/abs-2601-17261,DBLP:journals/corr/abs-2602-17155,ma2025revisiting,dang2025fzoo}, variance reduction and momentum techniques \cite{yuepseuzo,DBLP:conf/icml/GautamPZRH24,DBLP:conf/emnlp/PengLWC25}, or Hessian-informed preconditioning \cite{DBLP:conf/iclr/ZhaoDYD0T25,zhao2025helene}. Together, these methods demonstrate that ZO fine-tuning can be substantially improved by exploiting gradient structure, model architecture, and carefully designed perturbation distributions.

\paragraph{Subspace Optimization.}
Subspace-based optimization exploits the observation that useful update directions in large models often concentrate in a much lower-dimensional structure than the full parameter space \cite{DBLP:conf/icml/Zhao0CWAT24, DBLP:conf/icml/HeLHC025}. In the ZO setting, low-rank updates are frequently achieved through the use of randomly sampled projection matrices combined with lazy updates \cite{yu2025zeroth,DBLP:conf/iclr/ChenZCYW25,DBLP:journals/corr/abs-2602-17155}. More closely related to our work, AGZO \cite{DBLP:journals/corr/abs-2601-17261} takes an important step by constructing activation-guided subspaces for more informative ZO perturbations. Together, these works motivate a systematic study of the interaction between activation-induced spaces, large-query signal aggregation, and spectral update conditioning in ZO optimization.
\section{Method}
\label{seq:method}
\subsection{Motivation}
First-order fine-tuning benefits from more than explicit gradient access. In practice, effective FO updates exploit informative low-dimensional geometry, aggregate derivative signals across many tokens and parameters, and apply well-conditioned update transformations. Standard ZO methods lack these properties: they perturb the full parameter space, rely on very limited directional information per step, and usually apply unconditioned updates. StructZO is designed to recover these structural benefits in a forward-only optimization pipeline by estimating updates in an informative reduced space and applying spectral shaping before lifting them back to the original parameters.

\subsection{Preliminaries}
Zeroth-Order (ZO) optimization refers to methods that update parameters without explicit gradient information. In this work, we focus on the \textit{Random Gradient Estimator} (RGE), which samples a perturbation $\bz$ matching the parameter shape and queries a “noisy function value oracle” to compute updates. RGE admits both forward difference (left) and central difference (right) variants:
\begin{equation}
     \tilde\nabla_{\mu}f(\bx)=\frac{1}{q}\sum_{i=1}^q\frac{f(\bx+\mu \bz_i)-f(\bx)}{\mu}\bz_i,\
     \tilde\nabla_{\mu}f(\bx)=\frac{1}{q}\sum_{i=1}^q\frac{f(\bx+\mu \bz_i)-f(\bx-\mu \bz_i)}{2\mu}\bz_i,
\end{equation}
where $\bz_i\sim N(\bzero,\bI)$. By Taylor expansion, when $\mu\rightarrow0$, both tend to:
\begin{equation}
    \hat{\nabla} f(\bx) = \frac{1}{q}\sum_{i=1}^q\<\nabla f(\bx), \bz_i\>\bz_i,
\end{equation}
which is an unbiased estimate of $\nabla f(\bx)$ with a variance scaling with the parameter dimension as $\Theta\left(\frac{d}{q}\right)$ and is utilized in \textit{ZO-SGD} to update parameters like $\theta_{t+1}=\theta_{t}-\eta\hat\nabla f(\theta_t)$, where $\eta$ is the learning rate. For LLM fine-tuning, $d$ is extremely large, so applying large-scale multi-query ZO estimation directly in the full parameter space becomes computationally inefficient \cite{DBLP:conf/emnlp/PengLWC25}. This motivates estimating ZO updates in a lower-dimensional but informative search space.


\begin{figure*}[tb]
  \centering
  \begin{subfigure}[t]{0.48\textwidth}
    \centering
    \includegraphics[width=\textwidth]{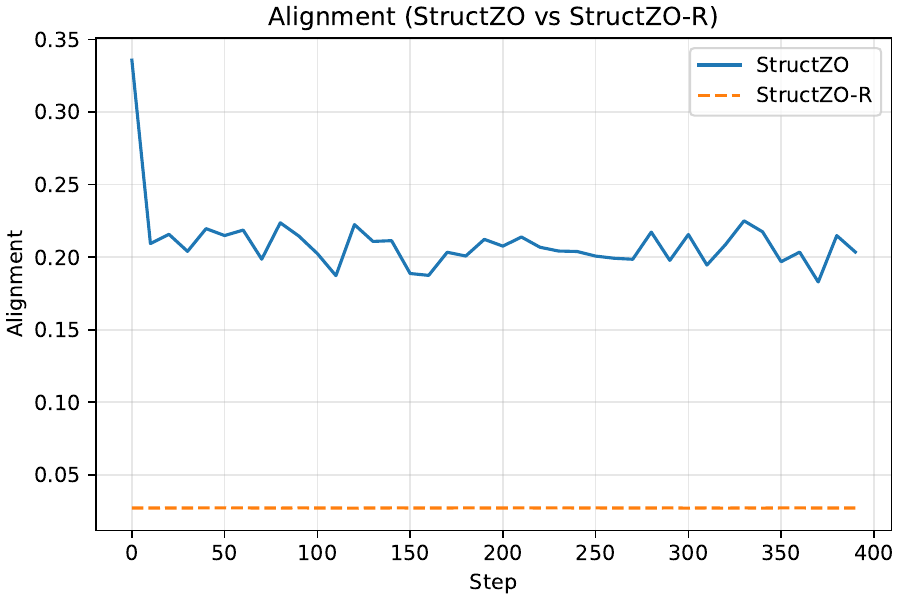}
    \caption{Alignment between $P$ and $V$}
    \label{fig:P, V alignment}
  \end{subfigure}
  \hfill
  \begin{subfigure}[t]{0.48\textwidth}
    \centering
    \includegraphics[width=\textwidth]{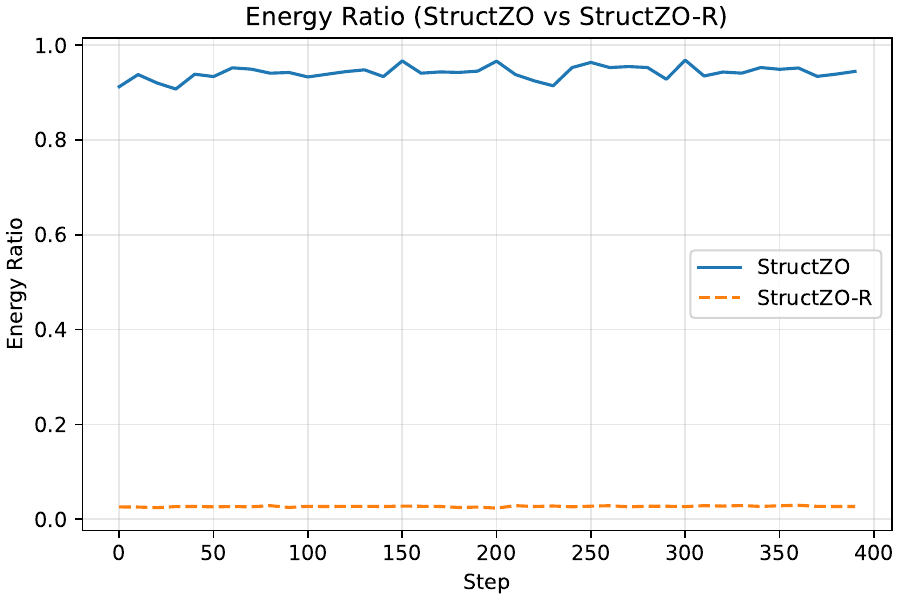}
    \caption{Energy ratio of $ GPP^\top$ in $G$}
    \label{fig:energy ratio}
  \end{subfigure}
  \hfill
  \begin{subfigure}[t]{0.48\textwidth}
    \centering
    \includegraphics[width=\textwidth]{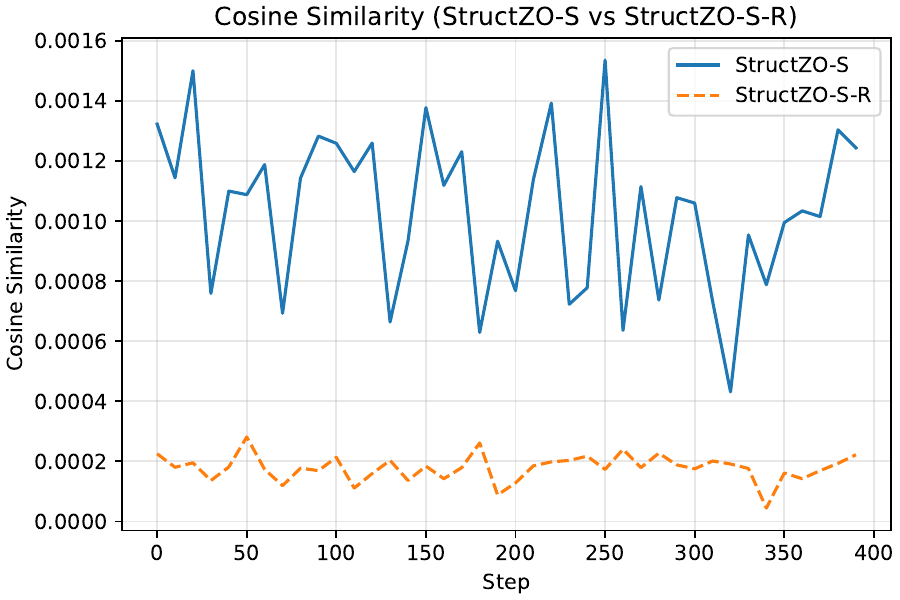}
    \caption{Cosine similarity between pre-muon $\hat G$ and $G$}
    \label{fig:cosine sim sgd}
  \end{subfigure}
  \hfill
  \begin{subfigure}[t]{0.48\textwidth}
    \centering
    \includegraphics[width=\textwidth]{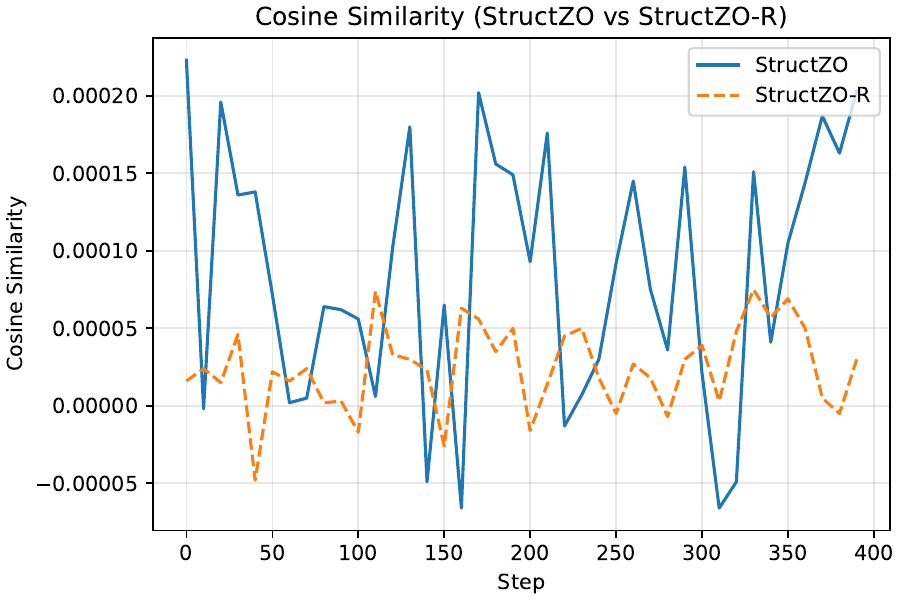}
    \caption{Cosine similarity between $\hat G$ and msign($G$)}
    \label{fig:cosine sim muon}
  \end{subfigure}
  \caption{Empirical geometry diagnostics of StructZO on SST-2, OPT1.3B. 
(a) Subspace alignment $\frac{1}{r}\|P^\top V_r\|_F^2$ between the constructed basis $P$ and the top-$r$ right singular subspace $V_r$ of the true gradient $G$. 
(b) Projected gradient energy $\|GPP^\top\|_F^2/\|G\|_F^2$. 
(c) Cosine similarity between the pre-spectral estimator and the raw gradient $G$. 
(d) Cosine similarity between the final spectrally shaped update and $\operatorname{msign}(G)$. The \textit{-R} suffix indicates that a random subspace is used to replace the original informative one, and the \textit{-S} suffix indicates that the Muon-style approach is not used.}
  \label{fig:geometry}
\end{figure*}

\subsection{StructZO}
The overall algorithm is summarized in Algorithm \ref{alg:szo}
\subsubsection{Activation-Induced Informative Subspaces}
Consider a Transformer linear layer $Y = AW^\top,$ where $A \in \mathbb{R}^{B \times n}$ denotes the input activation matrix of the current mini-batch and $W \in \mathbb{R}^{m\times n}$ is the trainable weight matrix. From back-propagation, $\nabla_WL=(\frac{\partial L}{\partial Y})^\top A:=\Delta^\top A$ and subsequently:
\begin{equation}
  (\nabla_WL)^\top \nabla_WL=  A^\top \Delta \Delta^\top A.
\end{equation}

When the output-side covariance $\Delta\Delta^\top$ is relatively isotropic or slowly varying, the right-side geometry of $\nabla_W L$ is largely governed by the activation Gram matrix $A^\top A$. This suggests that the dominant eigenspace of $A^\top A$ provides an informative coordinate system for ZO perturbations.


Instead of using isotropic random projections, StructZO therefore efficiently constructs an informative subspace directly from activations. Specifically, we sample a Gaussian sketch matrix $\Omega \in \mathbb{R}^{n \times (r+k)},$ where $r$ is the target rank and $k$ is an oversampling factor. We compute $Y = C\Omega = A^\top (A\Omega)$, followed by \textit{QR} decomposition $Y = QR$. The informative subspace basis is then defined as $\bP = Q_{[:,:r]} \in \mathbb{R}^{n \times r}$. Importantly, this construction is performed entirely from forward activations and does not require storing activations for backward propagation. The resulting projector $\Pi = \bP\bP^\top$
captures dominant activation directions and defines a reduced optimization space.
\subsubsection{Multi-Query Structured Zeroth-Order Estimation}
After constructing the informative subspace $\bP$, StructZO performs perturbations only inside the reduced coordinates. For each query, we sample $z_i \sim \mathcal{N}(0, I_{m\times r})$, and construct the lifted perturbation $\Delta_i = z_i P^\top$. Consequently, the aggregated gradient is:
\begin{equation}
\frac{1}{q}
\sum_{i=1}^q
\frac{
f(W+\mu z_iP^\top)-f(W)
}{\mu}
z_iP^\top.    
\end{equation}
Compared with full-space Gaussian ZO, the perturbation dimension is reduced from $mn$ to $mr$, significantly lowering estimator variance when $r \ll n$. Unlike random low-dimensional perturbations, the reduced coordinates here are constructed from the current forward activations and are therefore more aligned with optimization-relevant directions, achieving better query efficiency.

\subsubsection{Spectral Shaping via Reduced-Space Muon}
Although reduced-space perturbations substantially reduce estimation variance, the resulting updates may still suffer from anisotropic singular directions. To further stabilize optimization, we apply Muon-style spectral shaping inside the reduced coordinates and our final update direction is:
\begin{equation}
    O=\operatorname{msign}
(\frac{1}{q}\sum_{i=1}^q
\frac{L(W+\mu z_iP^\top)-L(W)}{\mu}
z_i)P^\top,
\end{equation}
where $\operatorname{msign}(G)=UV^\top$ for the SVD $G=U\Sigma V^\top$.
\begin{prop}
    \label{prop:sub+muon}
    For a matrix $G$ and its \textit{SVD} decomposition $G=U\Sigma V^\top$, if $P=V$ then $\operatorname{msign}(GP)P^\top=\operatorname{msign}(G)$.
\end{prop}
This operation equalizes active singular directions inside the informative subspace, producing well-conditioned updates analogous to spectrally normalized FO optimization. Moreover, according to Proposition \ref{prop:sub+muon}, StructZO does not merely approximate the raw gradient $G$; instead, it approximately reconstructs the spectrally conditioned update direction $\mathrm{msign}(G)$, which is typically more stable and optimization-effective than directly estimating $G$ itself.
\subsubsection{Empirical Geometry Diagnostics}
We further provide empirical diagnostics to validate whether the informative basis constructed from forward activations captures useful gradient geometry. In Figure~\ref{fig:geometry}, we compute true FO gradients only for analysis. Although the constructed basis $P$ does not exactly match the top right singular subspace $V_r$ of the true gradient, it shows substantially stronger alignment than a random basis. The projected energy ratio $\|GPP^\top\|_F^2/\|G\|_F^2$ is also much higher, indicating that the reduced space preserves a large fraction of gradient energy. Before spectral shaping, the structured estimator is closer to the raw gradient $G$ than its random-space counterpart. After spectral shaping, the final StructZO update is more aligned with the spectrally conditioned direction $\operatorname{msign}(G)$. These results support our design choice: the forward-induced basis provides a useful search space, while spectral shaping changes the optimization target from raw gradient approximation to a better-conditioned update direction.
\begin{algorithm}[h]
    \caption{StructZO}
    \label{alg:szo}
    \begin{algorithmic}[1]
        \Require Parameters $\{\boldsymbol{W_l}\}_{l=1}^L$, loss $L$, learning rate $\eta$, max iterations $T$, smoothing coefficient $\mu>0$, query number $q$, rank $r$, over-sample factor $k$ 
        \State\textbf{for} $t=0$ \textbf{to} $T$ \textbf{do}           \State\quad Sample base seed $s$ and batch $\mathcal{B}_t$
            \State\quad $\ell_0\leftarrow L(\{\boldsymbol{W_l}\}_{l=1}^L;\mathcal{B}_t)$  \quad \text{with activations}
            \State\quad Obtain informative subspace $\{P_l\}_{l=1}^L\leftarrow$ ExtractSubspace($\{A_l\}_{l=1}^L$, $r$, $k$) on-the-fly
            \State\quad\textbf{for} $i=1$ \textbf{to} $q$ \textbf{do}
                \State\qquad $s_i\leftarrow s+i$
                \State\qquad PerturbParameters($\{\boldsymbol{W_l}\}_{l=1}^L$, $\{P_l\}_{l=1}^L$, $\mu$, $s_i$)
                \State\qquad $\ell_i\leftarrow L(\{\boldsymbol{W_l}\}_{l=1}^L;\mathcal{B}_t)$
                \State\qquad PerturbParameters($\{\boldsymbol{W_l}\}_{l=1}^L$, $\{P_l\}_{l=1}^L$, $-\mu$, $s_i$)
            \State\quad\textbf{end for}
            \State\quad Initialize $\{g_l\}_{l=1}^L\leftarrow\textbf{0}$
            \State\quad\textbf{for} $i=1$ \textbf{to} $q$
            \textbf{do}
                \State\qquad Resample $\{z_l\}_{l=1}^L\leftarrow\text{Generator}(s_i)$
                \State\qquad $\{g_l\}_{l=1}^L\leftarrow\{g_l\}_{l=1}^L+\frac{\ell_i-\ell_0}{\mu}\{z_l\}_{l=1}^L$
            \State\quad\textbf{end for}
            \State\quad $\{\boldsymbol{W}_l\}_{l=1}^L\leftarrow\{\boldsymbol{W}_l-\eta \textbf{msign}(g_l)P_l^\top\}_{l=1}^L$
        \State\textbf{end for}
        \State \textbf{Function} ExtractSubspace($A$, $r$, $k$):
        \State\quad Let $d=A.\text{Shape[1]}$, Sample Gaussian sketch matrix $\Omega\in\mathbb{R}^{d\times (r+k)}$
        \State\quad $Y=A^\top (A\Omega)$, Perform \textit{QR} decomposition $Y=QR$
        \State\quad \textbf{Return} $P=Q[:,:r]$ \Comment{First $r$ columns of $Q$}
        \State \textbf{Function} PerturbParameters($\{\boldsymbol{W_l}\}_{l=1}^L$, $\{P_l\}_{l=1}^L$, $\mu$, $s$):
        \State\quad Reset random number generator seed $s$
        \State\quad \textbf{for} $l=1$ \textbf{to} $L$ \textbf{do}
            \State\qquad Sample low-rank Perturbation $z_l$
            \State\qquad $\boldsymbol{W_l} \leftarrow \boldsymbol{W_l}+\mu z_lP_l^\top$
        \State\quad\textbf{end for}
    \end{algorithmic}
\end{algorithm}

\subsection{MSE Analysis of StructZO}
We analyze the pre-spectral estimator to quantify the effect of informative reduced coordinates and multi-query aggregation. Consider one weight matrix $W\in\mathbb{R}^{m\times n}$ and stochastic loss $F_\xi(W)$, where $\xi$ denotes mini-batch randomness. Let
\[
f(W)=\mathbb{E}_{\xi}[F_\xi(W)],\qquad
G=\nabla f(W),\qquad
G_\xi=\nabla F_\xi(W).
\]
For each mini-batch, StructZO constructs $P_\xi\in\mathbb{R}^{n\times r}$ from the activation Gram matrix $A_\xi^\top A_\xi$ using the randomized range finder described above, and defines $\Pi_\xi=P_\xi P_\xi^\top$. Unlike information-agnostic random subspace methods, here $\Pi_\xi$ is
constructed from the same mini-batch activation $A_\xi$ and is therefore
statistically coupled with $\xi$. Thus, instead of enforcing
$\mathbb{E}[\Pi_\xi]=cI_n$, we directly analyze the coupled expectation
over $(\xi,\Omega)$. Given $q$ independent perturbations $z_i\sim\mathcal{N}(0,I_{m\times r})$, the pre-spectral estimator is
\[
\widehat{G}
=
\frac{1}{q}\sum_{i=1}^{q}
\frac{F_\xi(W+\mu z_iP_\xi^\top)-F_\xi(W)}{\mu}
z_iP_\xi^\top .
\]

\begin{theorem}[MSE bound for StructZO]
\label{thm:mse bound StructZO}
Under the assumptions \ref{ass:sigma^2}, \ref{ass:proj align}, \ref{ass:act gram} specified in Appendix~\ref{appendix:proof} and standard smoothness conditions on
$F_\xi$, the pre-Muon multi-query estimator satisfies
\[
\begin{aligned}
\mathbb{E}_{\xi,\Omega,z}\|\widehat G-G\|_F^2
&\le
\underbrace{
\frac{mr+1}{q}
\mathbb{E}_{\xi}\|G_\xi\|_F^2
}_{\text{query-induced variance}}
+
\underbrace{
\beta_G\alpha_{r,k}
\mathbb{E}_{\xi}
\sum_{j>r}\lambda_j(A_\xi^\top A_\xi)
}_{\text{subspace approximation bias}}
\\
&\quad+
\underbrace{\sigma^2}_{\text{stochastic noise}}
+
\underbrace{O(\mu^2)}_{\text{smoothing error}}.
\end{aligned}
\]
\end{theorem}

\begin{remark}[Interpretation]
The bound contains two key terms. The first term is the multi-query ZO
estimation variance inside the reduced space. Compared with full-space
Gaussian ZO, its dimension factor changes from $mn$ to $mr$, yielding a
variance reduction controlled by $r/n$ when $r\ll n$. The second term is
the activation-subspace projection bias. It is controlled by the spectral
tail of the activation Gram matrix $A_\xi^\top A_\xi$. Therefore, when the
activation spectrum has low effective rank, the informative
subspace can keep most useful gradient directions while avoiding
full-dimensional zeroth-order perturbations.
\end{remark}
\section{Experiments}
\label{seq:exp}
\paragraph{Overview.}
We evaluate StructZO on autoregressive LLMs across classification, multiple-choice, and generation tasks from GLUE \cite{wang2018glue} and SuperGLUE \cite{wang2019superglue}, comparing with representative ZO baselines, including MeZO \cite{malladi2023fine}, LOZO \cite{DBLP:conf/iclr/ChenZCYW25}, and HiZOO/HiZOO-L \cite{DBLP:conf/iclr/ZhaoDYD0T25}, as well as FO-Adam, on OPT-1.3B and Llama3.1-8B, and further test scalability on OPT-13B and OPT-30B. We report downstream performance, loss convergence, ablations, memory usage, and wall-clock time. Empirically, StructZO consistently outperforms existing ZO methods and substantially narrows the gap to FO-Adam. Its advantage comes from jointly leveraging informative subspaces, multi-query signal aggregation, and spectral shaping, which yields faster and deeper loss reduction while retaining the low-memory benefit of forward-only ZO optimization.

\paragraph{Setup.}
We implement all methods in the same training pipeline with identical data preprocessing, prompt templates, batching, and evaluation protocols. Following prior ZO fine-tuning work \cite{malladi2023fine,DBLP:conf/iclr/ChenZCYW25,DBLP:conf/iclr/ZhaoDYD0T25}, each task uses 1,024 training examples, 512 development examples, and 1,024 evaluation examples; for small datasets such as CB and COPA, we reserve 100 examples for validation and use the rest for training. Unless otherwise specified, all ZO methods use batch size 16. Single-query ZO baselines are trained for \textbf{20K} steps, while StructZO uses $q=100$ queries and is trained for \textbf{400} steps, giving a comparable forward-query scale but richer signal aggregation per update. FO-Adam is trained for 5 epochs with smaller task-dependent batch sizes when necessary. Hyperparameters are selected on the development set, and all experiments are conducted on a single NVIDIA H20 GPU with 96 GiB memory. Additional details are in Appendix~\ref{app:experiment}.
\subsection{Main Results}
\begin{table}[t]
  \centering
  \caption{Evaluation on OPT-1.3B and Llama3.1-8B using 1,024 training and 512 validation samples per task (for CB and COPA we reserve 100 validation samples and use the remainder for training). With the exception of the DROP task, where F1 is used, all other figures refer to accuracy (\%). Each model block lists zeroth-order (ZO) methods first, followed by the first-order (FO) baseline separated by a rule. \textbf{Bold} entries indicate the best score among ZO methods; FO rows are provided for reference.}
  \resizebox{\textwidth}{!}{
  \begin{tabular}{llrrrrrrrrrr}
    \toprule
    \textbf{Model} & \textbf{Method} & \textbf{SST-2} & \textbf{RTE} & \textbf{CB} & \textbf{BoolQ}  & \textbf{WiC} & \textbf{MultiRC} & \textbf{COPA} & \textbf{ReCoRD} & \textbf{DROP} & \textbf{AVG} \\
    \multicolumn{2}{l}{\textbf{Task type}} 
      & \multicolumn{6}{c}{\hrulefill classification \hrulefill} 
      & \multicolumn{2}{c}{\hrulefill multiple choice \hrulefill} 
      & \multicolumn{1}{c}{generation} 
      & \multicolumn{1}{c}{} \\
    \midrule
    \multirow{5}{*}{\textbf{OPT-1.3B}}
      & MeZO                & 88.8 & 66.4 & 67.8 & 63.4  & 57.9 & 56.2 & 73.0 & 70.7 & 22.4 & 63.0 \\
      & LOZO               & 89.4 & 61.3 & 67.8 & 60.8  & 55.1 & 58.6 & 73.0 & 71.2 & 21.9 & 62.1 \\
      & HiZOO & 89.6 & 61.7 & 67.8 & 63.1 & 60.2 & 58.9 & 73.0 & \textbf{71.8} & 22.5 & 63.2\\
      & StructZO                & \textbf{92.4} & \textbf{69.0} & \textbf{71.4} & \textbf{68.5}  & \textbf{62.6} & \textbf{63.9} & \textbf{74.0} & \textbf{71.8} & \textbf{26.7} & \textbf{66.7} \\
      \cmidrule(lr){2-12}
      & FO-Adam & 93.0 & 70.3 & 71.4 & 69.4  & 63.0 & 64.2 & 74.0 & 71.7 & 27.7 & 67.2 \\
    \midrule
    \midrule
    \multirow{6}{*}{\textbf{Llama3.1-8B}}
      & MeZO                & 90.2 & 70.7 & 69.6 & 79.1  & 58.7 & 78.1 & 90.0 & \textbf{84.9} & 62.2 & 75.9 \\
      & LOZO             & 91.3 & 70.7 & 73.2 & 80.2  & 58.7 & 73.8 & \textbf{91.0} & \textbf{84.9} & 62.7 & 76.3 \\
      & HiZOO-L & 90.2 & 71.2 & 71.4 & 79.6 & 55.1 & 77.8 & 90.0 & 84.8 & 58.5 & 75.4\\
      & StructZO                & \textbf{93.7} & \textbf{84.5} & \textbf{92.9} & \textbf{84.1}  & \textbf{66.8} & \textbf{83.2} & \textbf{91.0} & 84.8 & \textbf{65.5} & \textbf{82.9} \\
      \cmidrule(lr){2-12}
      & FO-Adam           & 94.7 & 88.0 & 91.0 & 87.9  & 70.4 & 82.5 & 91.0 & 84.0 & 67.2 & 84.1 \\
    \midrule
    \bottomrule
  \end{tabular}
  }
  \label{tab:main}
\end{table}
\begin{table}[t]
  \centering
  \renewcommand{\arraystretch}{0.85}
  \resizebox{0.7\textwidth}{!}{
  \begin{tabular}{llrrrrr}
    \toprule
    \textbf{Model} & \textbf{Method} & \textbf{SST-2} & \textbf{RTE} & \textbf{BoolQ}  & \textbf{WiC} & \textbf{AVG} \\
    \midrule
    \multirow{4}{*}{\textbf{OPT-13B}}
      & MeZO & 91.3 & 67.5 & 68.0 & 59.5 & 71.6  \\
      & LOZO & 91.7 & 70.4 & 71.9 & 60.8 & 73.7  \\
      & HiZOO-L & 92.1 & 68.2 & 66.6 & 59.4 & 71.6 \\
      & StructZO & \textbf{93.7} & \textbf{72.3} & \textbf{78.8} & \textbf{64.0} & \textbf{77.2} \\
    \midrule
    \multirow{4}{*}{\textbf{OPT-30B}}
     & MeZO & 90.7 & 67.1 & 67.4 & 56.3 & 70.4  \\
      & LOZO & 92.0 & 65.3 & 72.3 & 57.2 & 71.7  \\
      & HiZOO-L & 90.3 & 69.3 & 68.4 & 58.6 & 71.7 \\
      & StructZO & \textbf{93.2} & \textbf{79.5} & \textbf{78.7} & \textbf{60.4} & \textbf{78.0} \\
    \bottomrule
  \end{tabular}
  }
  \caption{Full-parameter ZO fine-tuning results on OPT-13B and OPT-30B.}
  \label{tab:OPT30B}
\end{table}

\begin{figure*}[tb]
  \centering
  \begin{subfigure}[t]{0.46\textwidth}
    \centering
    \includegraphics[width=\textwidth]{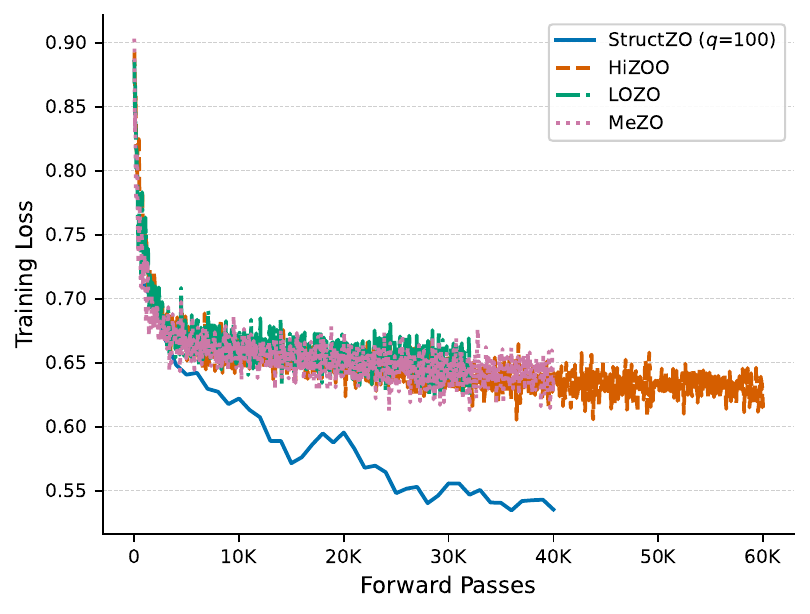}
    \caption{Training loss versus forward passes on BoolQ, OPT1.3B.}
    \label{fig:BoolQ_loss}
  \end{subfigure}
  \hfill
  \begin{subfigure}[t]{0.46\textwidth}
    \centering
    \includegraphics[width=\textwidth]{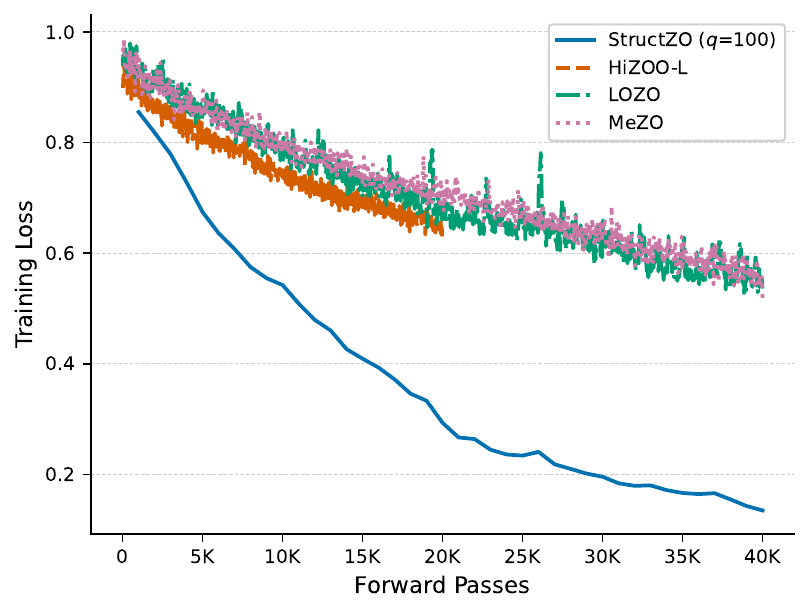}
    \caption{Training loss versus forward passes on CB, Llama3.1-8B.}
    \label{fig:CB_loss}
  \end{subfigure}

  \caption{Training loss versus forward passes for (a) BoolQ on OPT1.3B and (b) CB on Llama3.1-8B, comparing MeZO, LOZO ,HiZOO(-L) and StructZO.}
  \label{fig:loss curve}
\end{figure*}
Table~\ref{tab:main} reports the main results on OPT-1.3B and Llama3.1-8B. StructZO consistently achieves the best performance among all ZO methods across classification, multiple-choice, and generation tasks, with up to 8.7\% average improvement over the strongest ZO baseline and only about 1.5\% average gap to FO-Adam. The gains are particularly pronounced on challenging tasks such as CB on Llama3.1-8B, where StructZO improves over the best ZO baseline by 26.9\%, showing that the proposed structured update space can substantially narrow the ZO--FO performance gap. Figure~\ref{fig:loss curve} further compares training loss under the same forward-pass budget: StructZO not only reduces loss faster than MeZO, LOZO, and HiZOO/HiZOO-L, but also converges to a lower final loss, indicating stronger optimization effectiveness rather than merely better checkpoint selection. Finally, Table~\ref{tab:OPT30B} extends the comparison to OPT-13B and OPT-30B, where StructZO remains consistently superior to existing ZO baselines. Overall, these results suggest that StructZO scales favorably with model size and provides a practical forward-only alternative to FO fine-tuning.
\subsection{Ablation Study}
\begin{table}[t]
  \centering
  \caption{Ablation study of StructZO on OPT-1.3B and Llama3.1-8B across classification, multiple-choice, and generation tasks. StructZO-R replaces the activation-induced subspace with a random subspace, while StructZO-S replaces Muon-style spectral shaping with SGD-style updates.}
  \resizebox{\textwidth}{!}{
  \begin{tabular}{llrrrrrrrrrr}
    \toprule
    \textbf{Model} & \textbf{Method} & \textbf{SST-2} & \textbf{RTE} & \textbf{CB} & \textbf{BoolQ}  & \textbf{WiC} & \textbf{MultiRC} & \textbf{COPA} & \textbf{ReCoRD} & \textbf{DROP} & \textbf{AVG} \\
    \multicolumn{2}{l}{\textbf{Task type}} 
      & \multicolumn{6}{c}{\hrulefill classification \hrulefill} 
      & \multicolumn{2}{c}{\hrulefill multiple choice \hrulefill} 
      & \multicolumn{1}{c}{generation} 
      & \multicolumn{1}{c}{} \\
    \midrule
    \multirow{3}{*}{\textbf{OPT-1.3B}}
      & StructZO-R & 91.7 & 59.5 & 71.4 & 64.3  & 56.1 & 59.0 & 74.0 & 72.6 & 22.8 & 63.5 \\
      & StructZO-S              & 91.8 & 62.0 & 67.8 & 63.2 & 57.5 & 60.2 & 72.0 & 72.0 & 25.3 & 63.5 \\
      & StructZO                & 92.4 & 69.0 & 71.4 & 68.5  & 62.6 & 63.9 & 74.0 & 71.8 & 26.7 & 66.7 \\
    \midrule
    \midrule
    \multirow{3}{*}{\textbf{Llama3.1-8B}}
      & StructZO-R & 88.3 & 66.7 & 73.2 & 77.7  & 62.5 & 77.1 & 90.0 & 86.5 & 62.4 & 76.0 \\
      & StructZO-S              & 92.8 & 76.1 & 75.0 & 84.0  & 61.5 & 73.0 & 90.0 & 86.4 & 57.3 & 77.3 \\
      & StructZO                & 93.7 & 84.5 & 92.9 & 84.1  & 66.8 & 83.2 & 91.0 & 84.8 & 65.5 & 82.9 \\
    \midrule
    \bottomrule
  \end{tabular}
  }
  \label{tab:variants}
\end{table}

\begin{figure*}[tb]
  \centering
  \begin{subfigure}[t]{0.46\textwidth}
    \centering
    \includegraphics[width=\textwidth]{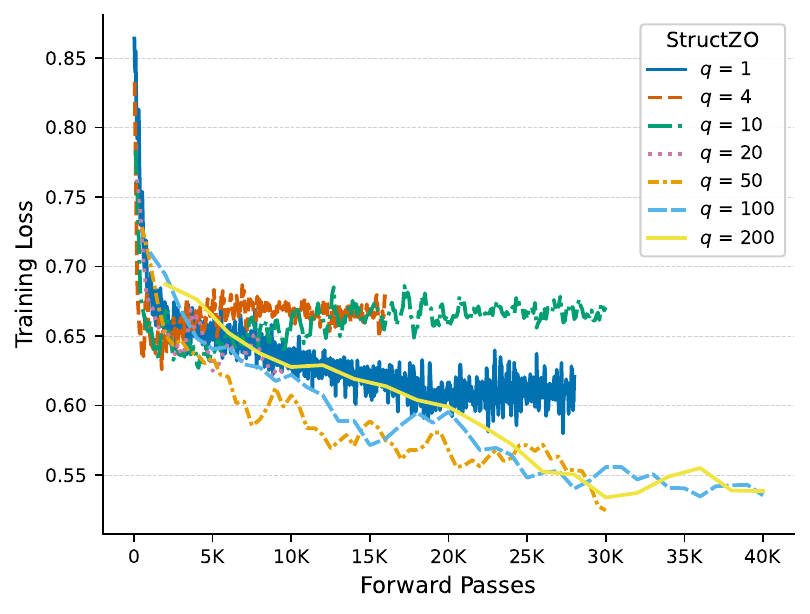}
    \caption{Training loss versus forward passes on BoolQ, OPT1.3B.}
    \label{fig:different q BoolQ}
  \end{subfigure}
  \hfill
  \begin{subfigure}[t]{0.46\textwidth}
    \centering
    \includegraphics[width=\textwidth]{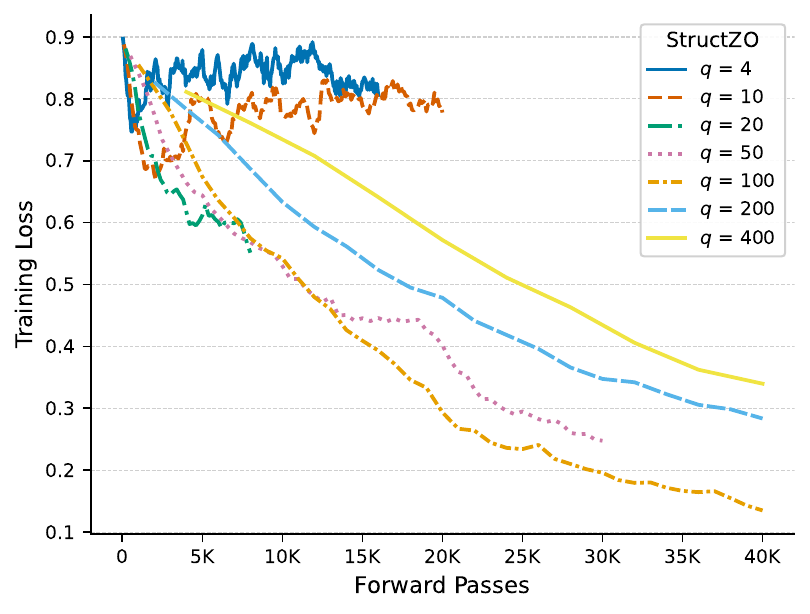}
    \caption{Training loss versus forward passes on CB, Llama3.1-8B.}
    \label{fig:different q CB}
  \end{subfigure}

  \caption{Training loss versus forward passes for (a) BoolQ on OPT1.3B and (b) CB on Llama3.1-8B, comparing different query number $q$.}
  \label{fig:different q}
\end{figure*}
Table~\ref{tab:variants} studies two ablated variants of StructZO: StructZO-R uses a random reduced space, while StructZO-S replaces spectral shaping with SGD-style updates. The full StructZO consistently outperforms both variants on OPT-1.3B and Llama3.1-8B. The clear drop from removing either part indicates that the effectiveness of StructZO comes from their joint interaction rather than from any single isolated component. Figure~\ref{fig:different q} further examines the effect of the query number $q$. Increasing $q$ leads to faster and lower loss reduction, suggesting that once queries are collected in a more informative reduced space, larger query budgets can be effectively converted into more accurate update directions, similar in spirit to the signal aggregation in FO optimization. Notably, although the SGD-style variant can be closer to the raw gradient $G$ shown in Figure ~\ref{fig:geometry}, it still lags behind the full StructZO, suggesting that the spectrally conditioned direction, rather than raw gradient approximation alone, is crucial for effective optimization.
\subsection{Memory and Wall-Clock Efficiency}

\paragraph{Memory overhead.}
Figure~\ref{fig:memory} compares peak GPU memory usage. StructZO avoids backpropagation, backward activation storage, and Adam optimizer states, so its only major extra memory over MeZO is the informative basis $P \in \mathbb{R}^{n \times r}$ for each weight matrix $W \in \mathbb{R}^{m \times n}$, costing $O(rn)$ storage. With the resampling trick, perturbations and reduced-space estimates are regenerated from random seeds, avoiding storage of all $q$ query directions. Therefore, StructZO adds only modest memory overhead compared with MeZO, avoids the per-parameter curvature states used by HiZOO-style methods, and still achieves up to $11.2\times$ memory reduction over FO-Adam.
\paragraph{Wall-clock time.}
Table \ref{tab:time} reports the wall-clock training time. Since StructZO uses multi-query estimation with $q=100$, each optimization step is slower than single-query ZO baselines. However, each step also aggregates richer directional information in an informative reduced space, allowing StructZO to use far fewer optimization steps. Despite its higher per-step cost, StructZO achieves lower or comparable total training time than MeZO and LOZO, and is substantially faster than HiZOO/HiZOO-L. Overall, StructZO trades higher per-step cost for much fewer and more effective updates, resulting in a favorable end-to-end efficiency.

\begin{figure}[t]
    \centering
    \renewcommand{\arraystretch}{2}
    \begin{minipage}[b]{0.45\textwidth}
        \includegraphics[width=\linewidth]{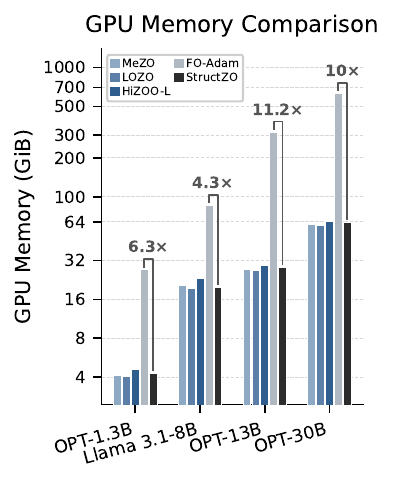}
        \caption{Memory overhead for different models under various ZO algorithms and FO-Adam. Details are shown in Appendix ~\ref{app:experiment}.}
        \label{fig:memory}
    \end{minipage}
    \hfill
    \begin{minipage}[b]{0.45\textwidth}
    \resizebox{\textwidth}{!}{
        \begin{tabular}{lllll}
        \toprule
        \textbf{Task} & \textbf{Method} & \textbf{Time / step (s)} & \textbf{Training steps} & \textbf{Total time (h)} \\
        \multirow{4}{*}{SST-2} & MeZO & 0.658 & 20K & 3.656 \\
        & LOZO & 0.658 & 20K & 3.654 \\
        & HiZOO-L & 1.435 & 20K & 7.970 \\
        & StructZO & 26.58 & 400 & 2.954 \\
        \midrule
        \multirow{4}{*}{RTE} & MeZO & 1.762 & 20K & 9.79 \\
        & LOZO & 1.762 & 20K & 9.79  \\
        & HiZOO-L & 3.653 & 20K & 20.30 \\
        & StructZO & 82.53 & 400 & 9.17 \\
        \midrule
        \multirow{4}{*}{BoolQ} & MeZO & 2.698 & 20K & 14.99 \\
        & LOZO & 2.698 & 20K & 14.99 \\
        & HiZOO-L & 5.507 & 20K & 30.60 \\
        & StructZO & 132.87 & 400 & 14.76 \\ 
        \bottomrule
        \end{tabular}
        }
         \captionof{table}{Wall-clock time for SST-2, RTE and BoolQ on Llama3.1-8B. We measured the time overhead for the first 100 steps and the first 10 steps for both ZO baselines and StructZO respectively, and thereby estimated their per-step and total execution times. }
    \label{tab:time}
    \end{minipage}
\end{figure}

\section{Conclusion}
\label{seq:conclusion}
We introduced StructZO, a structured zeroth-order optimization framework for memory-efficient LLM fine-tuning. By aggregating large-query directional signals in an informative subspace and applying spectral shaping in reduced coordinates, StructZO improves the effectiveness of forward-only optimization. As a limitation, StructZO depends on the forward-induced subspace capturing optimization-relevant directions. When this structure is weak, the benefit may decrease, as suggested by the projection-bias term in our analysis. In addition, our experiments focus on GLUE/SuperGLUE-style limited-data fine-tuning, and extending the evaluation to longer-context and more complex generation tasks remains promising future work.


\newpage
\bibliographystyle{IEEEtran}
\bibliography{references}

\newpage
\appendix
\section{Proof}
\label{appendix:proof}
\begin{proof}[Proof of Proposition~\ref{prop:sub+muon}]
Recall that the matrix sign operation used in Muon can be written as the polar factor
\[
    \operatorname{msign}(X) = X(X^\top X)^{-1/2},
\]
which is equivalent to $\operatorname{msign}(X)=U_XV_X^\top$ for the SVD $X=U_X\Sigma_XV_X^\top$.
Let the SVD of $G$ be
\[
    G = U\Sigma V^\top .
\]
When $P=V$, we have
\[
    GP = G V = U\Sigma V^\top V = U\Sigma .
\]
Therefore,
\[
\begin{aligned}
    \operatorname{msign}(GP)
    &= GP\big((GP)^\top(GP)\big)^{-1/2} \\
    &= U\Sigma
    \big((U\Sigma)^\top(U\Sigma)\big)^{-1/2} \\
    &= U\Sigma(\Sigma^\top\Sigma)^{-1/2}.
\end{aligned}
\]
Since $\Sigma$ is diagonal with nonnegative singular values, we have
\[
    \Sigma(\Sigma^\top\Sigma)^{-1/2} = I
\]
on the singular subspace. Hence
\[
    \operatorname{msign}(GP)=U.
\]
Multiplying by $P^\top=V^\top$ gives
\[
    \operatorname{msign}(GP)P^\top
    =
    U V^\top
    =
    \operatorname{msign}(G).
\]
This completes the proof.
\end{proof}

\paragraph{Interpretation as approximate low-rank spectral shaping.}
Proposition~\ref{prop:sub+muon} gives an exact identity when the reduced basis matches the right singular space of the gradient. In practice, the activation-induced basis only approximates this space. We therefore do not claim exact recovery of the full-space Muon update; instead, the empirical diagnostics in Figure~\ref{fig:geometry} show that the constructed basis captures substantially more gradient-relevant right-side structure than a random basis, which explains why applying spectral shaping in the reduced coordinates can be effective.

For Theorem \ref{thm:mse bound StructZO}, we use the following assumptions.
\begin{assumption}[Bounded stochastic gradient noise]
\label{ass:sigma^2}
There exists $\sigma^2>0$ such that
\[
\mathbb{E}_{\xi}\|G_\xi-G\|_F^2\le \sigma^2.
\]
\end{assumption}

\begin{assumption}[Projected alignment assumption]
\label{ass:proj align}
Let $C_\xi=A_\xi^\top A_\xi$. There exists $\beta_G>0$ such that
\[
\mathrm{tr}\!\left(G^\top G(I-\Pi_\xi)\right)
\le
\beta_G\,
\mathrm{tr}\!\left(C_\xi(I-\Pi_\xi)\right)
\]
for the activation-induced projector $\Pi_\xi$.
\end{assumption}

\begin{assumption}[Activation Gram sketchability]
\label{ass:act gram}
There exists $\alpha_{r,k}\ge 1$ such that the randomized QR subspace
satisfies
\[
\mathbb{E}_{\Omega}
\left[
\mathrm{tr}\!\left(C_\xi(I-\Pi_\xi)\right)\mid \xi
\right]
\le
\alpha_{r,k}
\sum_{j>r}\lambda_j(C_\xi),
\]
where $\lambda_1(C_\xi)\ge \cdots\ge \lambda_n(C_\xi)\ge 0$ are the
eigenvalues of $C_\xi$.
\end{assumption}

\subsection{A Coupled MSE Decomposition}

We first prove a decomposition for the pre-Muon estimator. Throughout
this section, we fix one weight matrix $W\in\mathbb{R}^{m\times n}$.
Let $F_\xi(W)$ be the stochastic loss, and define
\[
G_\xi=\nabla F_\xi(W),\qquad
G=\nabla f(W)=\mathbb{E}_\xi[G_\xi].
\]
Let $P_\xi=P(A_\xi,\Omega)\in\mathbb{R}^{n\times r}$ be the
activation-induced subspace obtained from $A_\xi^\top A_\xi\Omega$ by
QR decomposition, and let
\[
\Pi_\xi=P_\xi P_\xi^\top.
\]
We assume $P_\xi^\top P_\xi=I_r$, so that $\Pi_\xi$ is an orthogonal
projector.

\begin{prop}[Coupled MSE decomposition]
Let $z_1,\ldots,z_q\overset{\mathrm{i.i.d.}}{\sim}
\mathcal{N}(0,I_{mr})$ be independent of $(\xi,\Omega)$, and define
\[
\widehat G
=
\frac{1}{q}\sum_{i=1}^q
\frac{F_\xi(W+\mu z_iP_\xi^\top)-F_\xi(W)}{\mu}
z_iP_\xi^\top .
\]
Under standard smoothness conditions on $F_\xi$,
\[
\begin{aligned}
\mathbb{E}_{\xi,\Omega,z}\|\widehat G-G\|_F^2
&=
\frac{mr+1}{q}
\mathbb{E}_{\xi,\Omega}
\mathrm{tr}(G_\xi^\top G_\xi\Pi_\xi)
\\
&\quad+
\mathbb{E}_{\xi,\Omega}
\mathrm{tr}\!\left((G_\xi-G)^\top(G_\xi-G)\Pi_\xi\right)
\\
&\quad+
\mathbb{E}_{\xi,\Omega}
\mathrm{tr}\!\left(G^\top G(I-\Pi_\xi)\right)
+
O(\mu^2).
\end{aligned}
\]
\end{prop}

\begin{proof}
We decompose the finite-difference estimator into its first-order part
and a smoothing remainder. Define
\[
\widehat G^{(1)}
=
\frac{1}{q}\sum_{i=1}^q
\langle G_\xi,z_iP_\xi^\top\rangle z_iP_\xi^\top .
\]
Since
\[
\langle G_\xi,z_iP_\xi^\top\rangle
=
\langle G_\xi P_\xi,z_i\rangle,
\]
we can write
\[
\widehat G^{(1)}
=
\frac{1}{q}\sum_{i=1}^q
\langle G_\xi P_\xi,z_i\rangle z_iP_\xi^\top.
\]
The difference between $\widehat G$ and $\widehat G^{(1)}$ is the
finite-difference smoothing error. Under standard smoothness conditions,
its contribution to the squared error is $O(\mu^2)$.

We first analyze $\widehat G^{(1)}$. Conditioning on $(\xi,\Omega)$,
both $G_\xi$ and $P_\xi$ are fixed. Let
\[
H_\xi=G_\xi P_\xi\in\mathbb{R}^{m\times r}.
\]
Then
\[
\widehat G^{(1)}
=
\left(
\frac{1}{q}\sum_{i=1}^q
\langle H_\xi,z_i\rangle z_i
\right)P_\xi^\top.
\]
Because $z_i\sim\mathcal{N}(0,I_{mr})$, we have
\[
\mathbb{E}_z[\langle H_\xi,z_i\rangle z_i\mid \xi,\Omega]
=
H_\xi.
\]
Therefore,
\[
\mathbb{E}_z[\widehat G^{(1)}\mid \xi,\Omega]
=
H_\xi P_\xi^\top
=
G_\xi P_\xi P_\xi^\top
=
G_\xi\Pi_\xi.
\]

Next, we compute the conditional variance. Let
\[
h_\xi=\mathrm{vec}(H_\xi),\qquad u_i=\mathrm{vec}(z_i).
\]
Then $u_i\sim\mathcal{N}(0,I_{mr})$ and
\[
\mathrm{vec}\!\left(\langle H_\xi,z_i\rangle z_i\right)
=
(u_i^\top h_\xi)u_i.
\]
Using the Gaussian fourth-moment identity,
\[
\mathbb{E}\left[(u^\top h)^2 uu^\top\right]
=
\|h\|_2^2 I + 2hh^\top,
\]
we obtain
\[
\mathrm{Cov}\left((u^\top h)u\right)
=
\|h\|_2^2 I + hh^\top.
\]
Hence
\[
\mathbb{E}_z
\left[
\left\|
\frac{1}{q}\sum_{i=1}^q
\langle H_\xi,z_i\rangle z_i
-
H_\xi
\right\|_F^2
\middle| \xi,\Omega
\right]
=
\frac{mr+1}{q}\|H_\xi\|_F^2.
\]
Since $P_\xi^\top P_\xi=I_r$, right multiplication by $P_\xi^\top$
preserves the Frobenius norm of matrices in $\mathbb{R}^{m\times r}$:
\[
\|MP_\xi^\top\|_F^2=\mathrm{tr}(M P_\xi^\top P_\xi M^\top)=\|M\|_F^2.
\]
Thus,
\[
\mathbb{E}_z
\left[
\|\widehat G^{(1)}-G_\xi\Pi_\xi\|_F^2
\middle| \xi,\Omega
\right]
=
\frac{mr+1}{q}\|G_\xi P_\xi\|_F^2.
\]
Equivalently,
\[
\|G_\xi P_\xi\|_F^2
=
\mathrm{tr}(G_\xi^\top G_\xi\Pi_\xi).
\]

Now decompose the error to the full gradient $G$:
\[
\widehat G^{(1)}-G
=
(\widehat G^{(1)}-G_\xi\Pi_\xi)
+
(G_\xi\Pi_\xi-G).
\]
The cross term vanishes after conditioning on $(\xi,\Omega)$ because
\[
\mathbb{E}_z[\widehat G^{(1)}-G_\xi\Pi_\xi\mid \xi,\Omega]=0.
\]
Therefore,
\[
\begin{aligned}
\mathbb{E}_z
\left[
\|\widehat G^{(1)}-G\|_F^2
\middle| \xi,\Omega
\right]
&=
\frac{mr+1}{q}
\mathrm{tr}(G_\xi^\top G_\xi\Pi_\xi)
+
\|G_\xi\Pi_\xi-G\|_F^2.
\end{aligned}
\]

It remains to expand the second term. Let
\[
E_\xi=G_\xi-G.
\]
Then
\[
G_\xi\Pi_\xi-G
=
(G+E_\xi)\Pi_\xi-G
=
E_\xi\Pi_\xi-G(I-\Pi_\xi).
\]
Because $\Pi_\xi$ is an orthogonal projector,
\[
\Pi_\xi(I-\Pi_\xi)=0.
\]
Thus the two components are orthogonal in Frobenius inner product:
\[
\left\langle E_\xi\Pi_\xi,\; G(I-\Pi_\xi)\right\rangle_F
=
\mathrm{tr}\!\left(
\Pi_\xi E_\xi^\top G(I-\Pi_\xi)
\right)
=
0.
\]
Consequently,
\[
\|G_\xi\Pi_\xi-G\|_F^2
=
\|E_\xi\Pi_\xi\|_F^2
+
\|G(I-\Pi_\xi)\|_F^2.
\]
Equivalently,
\[
\|E_\xi\Pi_\xi\|_F^2
=
\mathrm{tr}\!\left(E_\xi^\top E_\xi\Pi_\xi\right),
\]
and
\[
\|G(I-\Pi_\xi)\|_F^2
=
\mathrm{tr}\!\left(G^\top G(I-\Pi_\xi)\right).
\]
Taking expectation over $(\xi,\Omega)$ and adding the smoothing remainder
$O(\mu^2)$ completes the proof.
\end{proof}

\subsection{Proof of Theorem~\ref{thm:mse bound StructZO}}

\begin{proof}
Starting from Proposition~\ref{prop:coupled-mse}, we bound the three
leading terms separately.

For the ZO query variance term, since $0\preceq \Pi_\xi\preceq I_n$,
\[
\mathrm{tr}(G_\xi^\top G_\xi\Pi_\xi)
=
\|G_\xi P_\xi\|_F^2
\le
\|G_\xi\|_F^2.
\]
Therefore,
\[
\frac{mr+1}{q}
\mathbb{E}_{\xi,\Omega}
\mathrm{tr}(G_\xi^\top G_\xi\Pi_\xi)
\le
\frac{mr+1}{q}
\mathbb{E}_{\xi}
\|G_\xi\|_F^2.
\]

For the stochastic gradient noise term, using again
$0\preceq \Pi_\xi\preceq I_n$, we have
\[
\mathrm{tr}\!\left((G_\xi-G)^\top(G_\xi-G)\Pi_\xi\right)
=
\|(G_\xi-G)P_\xi\|_F^2
\le
\|G_\xi-G\|_F^2.
\]
By the bounded stochastic gradient noise assumption,
\[
\mathbb{E}_{\xi,\Omega}
\mathrm{tr}\!\left((G_\xi-G)^\top(G_\xi-G)\Pi_\xi\right)
\le
\sigma^2.
\]

For the projection bias term, the activation-gradient alignment
assumption gives
\[
\mathrm{tr}\!\left(G^\top G(I-\Pi_\xi)\right)
\le
\beta_G
\mathrm{tr}\!\left(C_\xi(I-\Pi_\xi)\right).
\]
Taking conditional expectation over $\Omega$ and using activation Gram
sketchability,
\[
\mathbb{E}_{\Omega}
\left[
\mathrm{tr}\!\left(G^\top G(I-\Pi_\xi)\right)
\mid \xi
\right]
\le
\beta_G\alpha_{r,k}
\sum_{j>r}\lambda_j(C_\xi).
\]
Taking expectation over $\xi$ yields
\[
\mathbb{E}_{\xi,\Omega}
\mathrm{tr}\!\left(G^\top G(I-\Pi_\xi)\right)
\le
\beta_G\alpha_{r,k}
\mathbb{E}_{\xi}
\sum_{j>r}\lambda_j(C_\xi).
\]
Since $C_\xi=A_\xi^\top A_\xi$, combining the above inequalities gives
\[
\begin{aligned}
\mathbb{E}_{\xi,\Omega,z}\|\widehat G-G\|_F^2
&\le
\frac{mr+1}{q}
\mathbb{E}_{\xi}\|G_\xi\|_F^2
+
\beta_G\alpha_{r,k}
\mathbb{E}_{\xi}
\sum_{j>r}\lambda_j(A_\xi^\top A_\xi)
\\
&\quad+
\sigma^2
+
O(\mu^2).
\end{aligned}
\]
This completes the proof.
\end{proof}

\paragraph{On Projected alignment assumption.}

For a linear layer, the stochastic gradient often admits the form
\[
G_\xi=\Delta_\xi^\top A_\xi,
\]
where $A_\xi$ is the input activation and $\Delta_\xi$ is the output-side
error signal. Hence
\[
G_\xi^\top G_\xi
=
A_\xi^\top \Delta_\xi\Delta_\xi^\top A_\xi.
\]
If the output-side error satisfies
$\Delta_\xi\Delta_\xi^\top\preceq \beta_\xi I$, then
\[
G_\xi^\top G_\xi\preceq \beta_\xi A_\xi^\top A_\xi.
\]
This shows that the right-side gradient energy is controlled by the
activation Gram matrix. Therefore, directions with small activation Gram
energy are unlikely to carry large stable gradient energy, which motivates
the activation-gradient alignment assumption.

\paragraph{On activation Gram sketchability.}

The construction
\[
Y_\xi=A_\xi^\top(A_\xi\Omega)=C_\xi\Omega
\]
is a randomized range finder applied to the positive semidefinite
activation Gram matrix $C_\xi=A_\xi^\top A_\xi$. The QR step extracts an
orthonormal basis for the sketched range. The sketchability assumption
states that this randomized subspace captures the dominant eigenspace of
$C_\xi$ up to a multiplicative factor $\alpha_{r,k}$. The right-hand side
\[
\sum_{j>r}\lambda_j(C_\xi)
\]
is the optimal rank-$r$ residual trace by the Ky Fan principle. Thus,
when the activation Gram matrix has fast spectral decay, the residual
projection error is small.

\begin{remark}[Comparison with full-space and random-subspace ZO]
For full-space Gaussian ZO with the same query number $q$, the leading
variance term scales as
\[
\frac{mn+1}{q}\mathbb{E}_{\xi}\|G_\xi\|_F^2.
\]
In contrast, our activation-induced estimator has
\[
\frac{mr+1}{q}
\mathbb{E}_{\xi,\Omega}
\mathrm{tr}(G_\xi^\top G_\xi\Pi_\xi)
\le
\frac{mr+1}{q}
\mathbb{E}_{\xi}\|G_\xi\|_F^2.
\]
Thus, the ZO search dimension is reduced from $mn$ to $mr$.

For an isotropic random rank-$r$ projector $\Pi_{\rm rand}$ satisfying
$\mathbb{E}[\Pi_{\rm rand}]=\frac{r}{n}I_n$, the projection bias is
\[
\mathbb{E}_{\Pi_{\rm rand}}
\|G(I-\Pi_{\rm rand})\|_F^2
=
\left(1-\frac{r}{n}\right)\|G\|_F^2.
\]
When $r\ll n$, this random projection discards most of the full-gradient
energy. By contrast, our activation-induced projector has projection
bias controlled by
\[
\mathbb{E}_{\xi}\sum_{j>r}\lambda_j(A_\xi^\top A_\xi),
\]
which can be much smaller when the activation Gram matrix has low
effective rank.
\end{remark}

\newpage
\section{Experimental details}
\begin{table}[tbhp]
  \centering
  \caption{The hyperparameter grids for OPT-1.3B and Llama3.1-8B; the weight decay is set to zero. }
  \resizebox{0.77\textwidth}{!}{%
  \begin{tabular}{lrc}
    \toprule
    Experiment & Hyperparameters & Values\\
    \midrule
    MeZO & Batch size & 16\\
    & Learning rate&  $\left\{1e-6,1e-7\right\}$ \\
         & $\mu$ & $1e-3$\\
    \midrule
    HiZOO & Batch size & 16\\
    & Learning rate&  $\left\{1e-6,1e-7\right\}$ \\
         & $\mu$ & $1e-3$\\
         & $\alpha_{\text{HiZOO}}$ & $\left\{1e-7,1e-8\right\}$ \\
         \midrule
    HiZOO-L  &  Batch size & 16\\
    & Learning rate&  $\left\{1e-6,1e-7\right\}$ \\
         & $\mu$ & $1e-3$\\
         & $\alpha_{\text{HiZOO}}$ & $\left\{1e-9,1e-10\right\}$ \\
        \midrule
    LOZO & Batch size & 16\\
    & Learning rate&  $\left\{1e-6,1e-7\right\}$ \\
         & $\mu$ & $1e-3$\\
         & $r$ & $\left\{1,2\right\}$\\
         & Step interval & $\left\{50,100\right\}$\\
         \midrule
    StructZO & Batch size & 16\\
    & Learning rate&  $\left\{1e-2,1e-3\right\}$ \\
         & $\mu$ & $1e-3$\\
         & $k$ & 10 \\
         & $r$ & 64 \\
    \midrule
    StructZO-S &  Batch size & 16\\
    & Learning rate&  $\left\{1e-5,1e-6\right\}$ \\
         & $\mu$ & $1e-3$\\
         & $k$ & 10 \\
         & $r$ & 64 \\
        \midrule
    StructZO-R &  Batch size & 16\\
    & Learning rate&  $\left\{1e-2,1e-3\right\}$ \\
         & $\mu$ & $1e-3$\\
         & $k$ & 10 \\
         & $r$ & 64 \\
        \midrule
    FO-Adam (OPT-1.3B) & Batch size & 4,8 \\
     & Learning rate & $\left\{1e-4,5e-5,1e-5\right\}$\\
     \midrule
    FO-Adam (Llama3.1-8B) & Batch size & 1,2 \\
     & Learning rate & $\left\{1e-4,5e-5,1e-5\right\}$\\
    \bottomrule
    \\
  \end{tabular}
  }
  \label{tab:hyperparameter}
\end{table}

\begin{table}[t]
    \centering
    \caption{
Peak GPU memory usage (GiB) on MultiRC with batch size 1 and an average sequence length of approximately 400 tokens. 
HiZOO-L denotes the low-rank version of HiZOO. 
FO-Adam results for OPT-13B and OPT-30B are taken from MeZO \cite{malladi2023fine}. 
}
    \resizebox{\textwidth}{!}{
        \begin{tabular}{lccccc}
        \toprule
        \textbf{Model} & MeZO & LOZO & HiZOO-L & StructZO & FO-Adam  \\
        \midrule
        \textbf{OPT1.3B} & 4.13 &4.08 & 4.61  &4.29 & 27.14 \\
         \textbf{Llama3.1-8B} & 20.55 &19.59 & 23.35  & 20.8&  85.29 \\
         \textbf{OPT13B} & 27.27 &26.91 & 29.49  &28.33 & 316 \\
         \textbf{OPT30B} & 60.98 &60.08 & 64.64  & 63.67& 633 \\
        \bottomrule
        \end{tabular}
        }
    \label{tab:detail_memory}
\end{table}

\begin{table}[t]
    \centering
    \caption{Detailed wall-clock comparison on OPT-1.3B and Llama3.1-8B. For MeZO, LOZO, and HiZOO/HiZOO-L, we report the average time per optimization step measured over the first 100 steps and estimate the total training time using 20K steps. For StructZO, we report the average time per optimization step measured over the first 10 steps and estimate the total training time using 400 steps, matching the main experimental setting with $q=100$. }
    \resizebox{\textwidth}{!}{
        \begin{tabular}{llllll}
        \toprule
        \textbf{Model} & \textbf{Task} & \textbf{Method} & \textbf{Time / step (s)} & \textbf{Training steps} & \textbf{Total time (h)} \\
        \multirow{12}{*}{\textbf{OPT1.3B}} & \multirow{4}{*}{SST-2} & MeZO & 0.146 & 20K & 0.811 \\
        & & LOZO & 0.145 & 20K & 0.804 \\
        & & HiZOO & 0.290 & 20K & 1.613 \\
        & & StructZO & 6.890 & 400 & 0.766 \\
        \cmidrule(lr){2-6}
        & \multirow{4}{*}{RTE} & MeZO & 0.358 & 20K & 1.99 \\
        & & LOZO & 0.360 & 20K & 2.00 \\
        & & HiZOO & 0.717 & 20K & 3.98 \\
        & & StructZO & 17.59 & 400 & 1.96 \\ 
        \cmidrule(lr){2-6}
         & \multirow{4}{*}{BoolQ} & MeZO & 0.558 & 20K & 3.102 \\
        & & LOZO & 0.554 & 20K & 3.080 \\
        & & HiZOO & 1.112 & 20K & 6.178 \\
        & & StructZO & 28.13 & 400 & 3.125 \\ 
        \midrule\midrule
        \multirow{12}{*}{\textbf{Llama3.1-8B}} & \multirow{4}{*}{SST-2} & MeZO & 0.658 & 20K & 3.656 \\
        & & LOZO & 0.658 & 20K & 3.654 \\
        & & HiZOO-L & 1.435 & 20K & 7.970 \\
        & & StructZO & 26.58 & 400 & 2.954 \\
        \cmidrule(lr){2-6}
        & \multirow{4}{*}{RTE} & MeZO & 1.762 & 20K & 9.79 \\
        & & LOZO & 1.762 & 20K & 9.79  \\
        & & HiZOO-L & 3.653 & 20K & 20.30 \\
        & & StructZO & 82.53 & 400 & 9.17 \\
        \cmidrule(lr){2-6}
         & \multirow{4}{*}{BoolQ} & MeZO & 2.698 & 20K & 14.99 \\
        & & LOZO & 2.698 & 20K & 14.99 \\
        & & HiZOO-L & 5.507 & 20K & 30.60 \\
        & & StructZO & 132.87 & 400 & 14.76 \\ 
        \bottomrule
        \end{tabular}
        }
    \label{tab:time appendix}
\end{table}

\begin{table}[t]
    \centering
    \caption{
Component-wise wall-clock time breakdown of StructZO. 
The ZO forward evaluations dominate the per-step cost, while constructing the activation-induced subspace $P$ and applying the final ZO update introduce only modest overhead.
}
    \resizebox{\textwidth}{!}{
        \begin{tabular}{llccc}
        \toprule
        \textbf{Model} & \textbf{Task} & \textbf{Time / Subspace construction $P$ (\%)} & \textbf{Time / ZO forward (\%)} & \textbf{Time / ZO update (\%)} \\
        \multirow{3}{*}{\textbf{OPT1.3B}} & SST-2 & 5.43 &  83.91 & 10.66 \\
        & RTE & 6.37 & 89.55 & 4.08\\
        & BoolQ & 6.25 & 91.15 & 2.60\\
        \midrule\midrule
        \multirow{3}{*}{\textbf{Llama3.1-8B}} & SST-2 & 4.95 & 92.77 & 2.28 \\
        & RTE & 6.03& 93.26 & 0.71 \\
        & BoolQ &6.20 & 93.32 & 0.48\\
        \bottomrule
        \end{tabular}
        }
    \label{tab:time component}
\end{table}

\begin{table}[t]
    \centering
    \caption{
Performance of the reduced-space rank $r$ in StructZO. 
}
    \resizebox{\textwidth}{!}{
        \begin{tabular}{llccccc}
        \toprule
        \textbf{Model} & \textbf{Task} & $r=16$ & $r=32$ & $r=64$ & $r=128$ & $r=256$ \\
        \midrule
        \textbf{OPT1.3B} & RTE & 63.2 & 67.8  & 69.0&66.1 & 68.9 \\
        \midrule
         \textbf{Llama3.1-8B} & SST-2 & 94.0  & 93.8  & 93.7& 93.8& 92.8 \\
        \bottomrule
        \end{tabular}
        }
    \label{tab:rank_r}
\end{table}

\begin{table}[t]
    \centering
    \caption{
Performance of StructZO with different query numbers $q$ on OPT-1.3B BoolQ and Llama3.1-8B CB. 
All results use the same StructZO pipeline while varying the number of ZO queries aggregated at each optimization step. 
}
    \resizebox{\textwidth}{!}{
        \begin{tabular}{llcccccc}
        \toprule
        \textbf{Model} & \textbf{Task} & $q=4$ & $q=10$ & $q=20$ & $q=50$ & $q=100$ & $q=200$ \\
        \midrule
        \textbf{OPT1.3B} & BoolQ & 63.5 & 62.4  & 65.3& 66.9 & 68.5 &  66.1 \\
        \midrule
         \textbf{Llama3.1-8B} & CB & 64.3  & 66.1  & 69.7 & 80.4 & 92.9 & 78.6\\
        \bottomrule
        \end{tabular}
        }
    \label{tab:diffenrent_q}
\end{table}
\label{app:experiment}
\paragraph{Datasets.}
We evaluate on classification, multiple-choice and generative tasks drawn from GLUE and SuperGLUE \cite{wang2018glue,wang2019superglue}. The individual tasks are SST-2 \cite{socher2013recursive}, RTE \cite{dagan2005pascal}, CB \cite{de2019commitmentbank}, BoolQ \cite{clark2019boolq}, WiC \cite{pilehvar2018wic}, MultiRC \cite{khashabi2018looking}, COPA \cite{roemmele2011choice}, ReCoRD \cite{zhang2018record} and DROP \cite{dua2019drop}. For each task, we randomly select 1024 training samples and 512 validation samples. For small datasets with very limited examples (CB and COPA) we reserve 100 examples for validation and use the remainder for training. We also adopt the same prompt design approach as MeZO \cite{malladi2023fine}, following GPT-3 and PromptSource with minor modifications. All experiments are conducted on a single NVIDIA H20 96GiB GPU.

\paragraph{Hyperparameters.}
Unless otherwise noted, each Zeroth-Order (ZO) algorithm is trained for 20K optimization steps and FO-Adam is trained for 5 epochs. For tokenization, we set a uniform maximum length $\text{MAX\_LEN}=2048$ which can cover the vast majority of language sequences across all tasks. For ZO methods we use a batch size of 16, and for FO-Adam we use the largest batch size that fits the GPU under identical system settings. To ensure a fair comparison, hyperparameters such as learning rate $\eta$ is selected by a small grid search and perturbation length $\mu$ is set $\mu=1e-3$ as default for each Zeroth-Order algorithm. Specifically, for HiZOO and HiZOO-L, we sweep Hessian smoothing coefficients and decay factors, and for LOZO we sweep the low-rank dimension $r$ and perturbation interval. For StructZO, the default oversampling factor used is $k=10$, and the low-rank subspace rank is $r=64$. Final hyperparameter settings are listed in Table~\ref{tab:hyperparameter}.

\paragraph{Detailed memory comparison.}
Table~\ref{tab:detail_memory} reports the detailed peak GPU memory usage on MultiRC with batch size 1 and an average sequence length of about 400 tokens. StructZO maintains memory usage close to inference-level ZO baselines across all tested model scales. Compared with MeZO and LOZO, StructZO introduces only a small additional memory cost, e.g., 4.29 GiB versus 4.13 GiB on OPT-1.3B, 20.80 GiB versus 20.55 GiB on Llama3.1-8B, and 63.67 GiB versus 60.98 GiB on OPT-30B. 
This modest overhead mainly comes from storing the activation-induced basis $P$, while the query perturbations are regenerated through the resampling trick rather than stored explicitly. In contrast, FO-Adam requires substantially larger memory due to backpropagation activations and optimizer states. StructZO reduces peak memory from 27.14 GiB to 4.29 GiB on OPT-1.3B, from 85.29 GiB to 20.80 GiB on Llama3.1-8B, from 316 GiB to 28.33 GiB on OPT-13B, and from 633 GiB to 63.67 GiB on OPT-30B. These results show that StructZO preserves the memory advantage of forward-only ZO optimization while scaling to substantially larger models than standard FO fine-tuning under the same memory budget.

\paragraph{Wall-clock time.}
Table~\ref{tab:time appendix} provides a detailed wall-clock comparison on OPT-1.3B and Llama3.1-8B across SST-2, RTE, and BoolQ. For single-query ZO baselines, we measure the average time over the first 100 optimization steps and estimate the total training time using 20K steps. For StructZO, we measure the average time over the first 10 optimization steps and estimate the total training time using 400 steps, following our main experimental setting with $q=100$. As expected, StructZO has a substantially larger per-step cost because each update performs large-query estimation and constructs structured perturbations in the activation-induced subspace. However, this cost should be interpreted together with the number of optimization steps: each StructZO step aggregates much richer directional information, allowing the method to use 50$\times$ fewer update steps than the single-query baselines. As a result, StructZO achieves comparable or lower end-to-end training time than MeZO and LOZO in most settings, and is consistently much faster than HiZOO/HiZOO-L. For example, on Llama3.1-8B, StructZO reduces the estimated total time from 3.656h to 2.954h on SST-2, from 9.79h to 9.17h on RTE, and from 14.99h to 14.76h on BoolQ compared with MeZO. These results show that although StructZO trades higher per-step latency for more informative updates, the resulting reduction in required optimization steps leads to favorable overall wall-clock efficiency.

\paragraph{Time breakdown of StructZO.}
To further understand where the computational cost of StructZO comes from, Table~\ref{tab:time component} reports the per-step wall-clock time breakdown into subspace construction, ZO forward evaluation, and ZO update. Across both OPT-1.3B and Llama3.1-8B, the dominant cost consistently comes from the ZO forward evaluations, which account for 83.91\%--93.32\% of the total step time. In contrast, constructing the activation-induced basis $P$ only contributes about 4.95\%--6.37\% of the step time, showing that the randomized subspace extraction introduces a relatively small overhead. The final ZO update, including reduced-space aggregation and spectral shaping, takes an even smaller fraction, especially on Llama3.1-8B where it is below 2.3\% for all three tasks. These results indicate that the main computational cost of StructZO is not the structured subspace construction or the spectral update itself, but rather the large-query forward evaluations. This also suggests that the proposed structure adds only modest algorithmic overhead beyond the necessary forward-query cost of multi-query ZO estimation.

\paragraph{Effect of query number.}
Table~\ref{tab:diffenrent_q} reports the downstream performance of StructZO under different query numbers $q$. 
The query number controls how many directional evaluations are aggregated within the activation-induced reduced space at each optimization step. 
When $q$ is small, the ZO estimator remains relatively noisy even though the search space is structured, leading to weaker performance. 
For example, on OPT-1.3B BoolQ, increasing $q$ from $4$ to $100$ improves accuracy from $63.5$ to $68.5$; on Llama3.1-8B CB, the improvement is even more pronounced, from $64.3$ to $92.9$. 
This trend is consistent with the training-loss curves in Figure~\ref{fig:different q}, where larger query budgets lead to faster and deeper loss reduction under the same forward-pass view. 
However, increasing $q$ beyond a certain point does not necessarily further improve downstream performance under a fixed computational budget. Overall, $q=100$ provides a strong balance between estimator quality, update frequency, and end-task performance, and we therefore use it as the default setting.

\paragraph{Effect of subspace rank.}
Table~\ref{tab:rank_r} studies the effect of the reduced-space rank $r$ in StructZO. 
The rank $r$ is an important hyperparameter because it directly controls the dimension and expressiveness of the activation-induced search space. 
When $r$ is too small, the reduced space may fail to preserve enough optimization-relevant directions, leading to a clear performance drop. 
For example, on OPT-1.3B RTE, increasing $r$ from $16$ to $64$ improves accuracy from $63.2$ to $69.0$, showing that a sufficiently expressive subspace is important for StructZO. 
However, once the rank is large enough, further increasing $r$ does not necessarily improve performance: using $r=128$ or $r=256$ gives comparable or slightly worse results on OPT-1.3B RTE, and the performance on Llama3.1-8B SST-2 remains relatively stable for $r\in\{16,32,64,128\}$ with a mild drop at $r=256$. 
This suggests that $r$ plays a meaningful role in balancing subspace expressiveness and ZO estimation efficiency. 
The default choice $r=64$ provides a strong trade-off: it is large enough to capture useful activation-induced directions, while avoiding the unnecessary increase in effective ZO dimension caused by overly large ranks.


\newpage

\end{document}